\documentclass[twoside]{article}
\pdfoutput=1
\usepackage[accepted]{aistats2020}

\usepackage{graphicx}

\usepackage[utf8]{inputenc} 
\usepackage[T1]{fontenc}    
\usepackage{url}            
\usepackage{amsfonts}       
\usepackage{nicefrac}       
\usepackage{microtype}      
\usepackage{tabu}
\usepackage[font=small]{caption}
\usepackage[round]{natbib}
\usepackage[round]{natbib}

\usepackage{hyperref}
\usepackage{mathtools}
\usepackage{amssymb}
\usepackage{amsthm}
\usepackage{paralist}
\usepackage[dvipsnames]{xcolor}
\usepackage[nameinlink]{cleveref} 
\usepackage[acronym,nowarn,section,nogroupskip,nonumberlist]{glossaries}
\usepackage{todonotes}
\usepackage{thmtools}
\usepackage{thm-restate}
\usepackage{pifont}
\usepackage{multirow}
\usepackage{url}

\usepackage{subcaption}

\definecolor{red}{HTML}{E41A1C}
\definecolor{orange}{HTML}{FF7F00}
\definecolor{yellow}{HTML}{FFC020}
\definecolor{green}{HTML}{4DAF4A}
\definecolor{blue}{HTML}{377EB8}
\definecolor{purple}{HTML}{984EA3}
\graphicspath{{fig/}}
\hypersetup{%
	colorlinks,
	linkcolor={violet!70!black},
	citecolor={YellowOrange!70!black},
	urlcolor={Aquamarine!70!black}
}
\theoremstyle{plain}

\theoremstyle{remark}

\Crefname{algocf}{Algorithm}{Algorithms}
\crefname{algorithm}{Algorithm}{Algorithms}
\crefname{figure}{Figure}{Figure}
\crefname{section}{§}{§§}
\Crefname{section}{§}{§§}
\crefformat{equation}{(#2#1#3)}
\glsdisablehyper{}
\newacronym{KL}{kl}{Kullback-Leibler}
\newacronym{SGD}{sgd}{stochastic gradient descent}
\newacronym{OED}{BOED}{Bayesian optimal experimental design}
\newacronym{MI}{MI}{mutual information}
\newacronym{EIG}{EIG}{expected information gain}
\newacronym{NMC}{NMC}{nested Monte Carlo}
\newacronym{ACE}{ACE}{adaptive contrastive estimation}
\newacronym{NCE}{NCE}{noise contrastive estimation}
\newacronym{BA}{BA}{Barber-Agakov}
\newacronym{BO}{BO}{Bayesian optimization}
\newacronym{MC}{MC}{Monte Carlo}
\newacronym{RMSE}{RMSE}{root mean squared error}
\presetkeys{todonotes}{%
	backgroundcolor=blue!10!white,
	linecolor=blue!10!white,
	bordercolor=blue!10!white
}{}

\definecolor{Green}{RGB}{50,200,50}

\usepackage{scalerel}
\usepackage{mleftright}
\usepackage{bm}

\DeclareMathOperator{\E}{{}\mathbb{E}}

\DeclareMathOperator*{\argmax}{arg\,max} 

\renewcommand{\to}{\ensuremath{\rightarrow}}              

\newcommand{\entropy}[1]{H\!\left[#1\right]}

\newcommand{\iid}{\overset{\scriptstyle{\text{i.i.d.}}}{\sim}}
\newcommand{\ppx}[1]{\frac{\partial}{\partial #1}}
\newcommand{\pypx}[2]{\frac{\partial #1}{\partial #2}}

\usepackage[T1]{fontenc}
\usepackage{lmodern}

\begin{document}

\runningauthor{Adam Foster, Martin Jankowiak, Matthew O'Meara, Yee Whye Teh, Tom Rainforth}

\twocolumn[

\aistatstitle{A Unified Stochastic Gradient Approach to Designing Bayesian-Optimal Experiments}

\aistatsauthor{ Adam Foster\textsuperscript{\textdagger} \And \hspace{-10pt}Martin Jankowiak\textsuperscript{\ddag} \And Matthew O'Meara\textsuperscript{\textsection} \And Yee Whye Teh\textsuperscript{\textdagger} \And  \hspace{-10pt}Tom Rainforth\textsuperscript{\textdagger \maltese}}
\vspace{.5ex}

\aistatsaddress{ \textsuperscript{\textdagger}Department of Statistics, University of Oxford, Oxford, UK \\
	\textsuperscript{\ddag}Uber AI, San Francisco, CA, USA \\
	\textsuperscript{\textsection}University of Michigan, Ann Arbor, MI, USA\\
	\textsuperscript{\maltese}Christ Church, University of Oxford, Oxford, UK \\ \texttt{adam.foster@stats.ox.ac.uk}}
]

\begin{abstract}
We introduce a fully stochastic gradient based approach to \acrfull{OED}.
Our approach utilizes variational lower bounds on the \acrfull{EIG} of an experiment that can be simultaneously optimized with respect to both the variational and design parameters.
This allows the design process to be carried out through a single unified stochastic gradient ascent procedure, in contrast to existing approaches that typically construct a pointwise \acrshort{EIG} estimator, before passing this estimator to a separate optimizer.
We provide a number of different variational objectives including the novel \acrfull{ACE} bound.
Finally, we show that our gradient-based approaches are able to provide effective design optimization in substantially higher dimensional settings than existing approaches.
\end{abstract}

\glsresetall

\section{INTRODUCTION}
\label{sec:introduction}

The design of experiments is a key problem in almost every scientific discipline. 
Namely, one wishes to construct an experiment that is most informative about the investigated process, while minimizing its cost.
For example, in a psychological trial, we want to ensure questions posed to participants are pertinent and do not have predictable responses.
In a pharmaceutical trial, we want to minimize the number of participants needed to test our hypotheses.
In an online automated help system, we want to ensure we ask questions that identify the user's problem as quickly as possible.

In all these scenarios, our ultimate high-level aim is to choose designs that maximize the information gathered by the experiment.
A powerful and broadly used approach for formalizing this aim is~\acrfull{OED} \citep{chaloner1995,lindley1956,myung2013}.
In \acrshort{OED}, we specify a Bayesian model for the experiment and then choose the design that maximizes the \acrfull{EIG} from running it.
More specifically, let $\theta$ denote the latent variables we wish to learn about from running the experiment and let $\xi \in \Xi$ represent the experimental design. 
By introducing a prior $p(\theta)$ and a predictive distribution $p(y|\theta,\xi)$ for experiment outcomes $y$, we can calculate the \acrshort{EIG} under this model by taking the expected reduction in posterior entropy
\begin{equation}
I(\xi) \triangleq \E_{p(y|\xi)}\left[ H[p(\theta)] - H[p(\theta|y,\xi)] \right],
\end{equation}
where $H[\cdot]$ represents the entropy of a distribution and $p(\theta|y,\xi) \propto p(\theta)p(y|\theta,\xi)$.
Our experimental design process now becomes that of the finding the design $\xi^*$ that maximizes $I(\xi)$.

Unfortunately, finding $\xi^*$ is typically a very challenging problem in practice. 
Even evaluating $I(\xi)$ for a single design is computationally difficult because it represents a \emph{nested} expectation and thus has no direct \acrlong{MC} estimator~\citep{nmc,zheng2018robust}.
Though a large variety of approaches for performing this estimation have been suggested~\citep{myung2013,watson2017quest+,kleinegesse2018efficient,foster2019variational}, the resulting \acrshort{OED} strategies share a critical common feature: they estimate $I(\xi)$ on a point-by-point basis and feed this estimator to an outer-level optimizer that selects the design.

This framework can be highly inefficient for a number of reasons.
For example, it adds an extra level of nesting to the overall computation process: $I(\xi)$ must be separately estimated for each $\xi$, substantially increasing the overall computational cost.
Furthermore, one must typically resort to gradient-free methods to carry out the resulting optimization, which means it is difficult to scale the overall \acrshort{OED} process to high dimensional design settings due to a dearth of optimization schemes which remain effective in such settings.

To alleviate these inefficiencies and open the door to applying \acrshort{OED} in high-dimensional settings, we introduce an alternative to this two-stage framework by introducing unified objectives that can be directly maximized to simultaneously estimate $I(\xi)$ and optimize $\xi$.
Specifically, by building on the work of~\citet{foster2019variational}, we construct variational lower bounds to $I(\xi)$ that can be simultaneously optimized with respect to both the variational and design parameters.
Optimizing the former ensures that we achieve a tight bound that in turn gives accurate estimates of $I(\xi)$, while simultaneously optimizing the latter circumvents the need for an expensive outer optimization process.
Critically, this approach allows the optimization to be performed using stochastic gradient ascent (SGA)~\citep{robbins1951stochastic} and therefore scaled to substantially higher dimensional design problems than existing approaches.

To account for the varying needs of different problem settings, we introduce several classes of suitable variational lower bounds.
Most notably, we introduce the \acrfull{ACE} bound: an~\acrshort{EIG} variational lower bound that can be made arbitrarily tight, while remaining amenable to simultaneous SGA on both the variational parameters and designs.

We demonstrate\footnote{Supporting code is provided at \path{https://github.com/ae-foster/pyro/tree/sgboed-reproduce}.} the applicability of our unified gradient approach using a wide range of experimental design problems, including a real-world high-dimensional example from the pharmacology literature~\citep{lyu2019ultra}.
We find that our approaches are able to effectively optimize the \acrshort{EIG}, consistently outperforming baseline two-stage approaches, with particularly large gains achieved for high-dimensional problems. These gains lead, in turn, to improved designs and more informative experiments.



\section{BACKGROUND}
\label{sec:background}
\subsection{Bayesian optimal experimental design}

When experimentation is costly, time consuming, or dangerous, it is essential to design experiments to learn the most from them. To choose between potential designs, we require a metric of the quality of a candidate design. In the \acrshort{OED} framework dating back to \citet{lindley1956}, this metric represents how much more certain we will become in our knowledge of the world after doing the experiment and analyzing the data. We prefer designs that will lead to strong conclusions even if we are not yet sure what those conclusions will be.

Specifically, we consider an experiment with design $\xi$, latent variable $\theta$ and outcome $y$. For example, $\xi$ may represent the question posed to a participant in a psychology trial, $y$ their answer, and $\theta$ their underlying psychological characteristic which is being studied. The \acrshort{OED} framework begins with a Bayesian model of the experimental process. This model consists of a likelihood $p(y|\theta,\xi)$ that predicts the experimental outcome under design $\xi$ and latent variable $\theta$ and a prior $p(\theta)$ which incorporates initial beliefs about the unknown $\theta$. After conducting the experiment, our beliefs about $\theta$ are updated to the posterior $p(\theta|y,\xi)$. The information gained about $\theta$ from doing the experiment with design $\xi$ and obtaining outcome $y$ is the reduction in entropy from the prior to the posterior
\begin{equation}
	\text{IG}(y, \xi) = H[p(\theta)] - H[p(\theta|y,\xi)].
\end{equation}
As it stands, information gain cannot be evaluated until after the experiment. To define a metric that will let us choose between designs before experimentation, we can use the \textit{expected} information gain (EIG), $I(\xi)$, by taking the expectation of IG over hypothesized outcomes $y$ using the marginal distribution under our model, $p(y|\xi)$, to give
\begin{equation}
	I(\xi) \triangleq \E_{p(y|\xi)}\left[ H[p(\theta)] - H[p(\theta|y,\xi)] \right]
\end{equation}
which can be rewritten in the form of a mutual information between $\theta$ and $y$ with $\xi$ fixed, namely
\begin{equation}
    I(\xi) = \text{MI}_\xi(\theta;y) = \E_{p(\theta)p(y|\theta,\xi)}\left[\log \frac{p(y|\theta,\xi)}{p(y|\xi)} \right].
    \label{eq:eig}
\end{equation}
The Bayesian optimal design, $\xi^*$, is now the one which maximizes EIG over the set of feasible designs $\Xi$
\begin{align}
	\xi^* = \argmax_{\xi \in \Xi} \; I(\xi).
	\label{eq:optim_problem}
\end{align}
In \textit{iterated} experimental design, we design a sequence $\xi_1, ..., \xi_T$ of experiments. At time $t$, the prior $p(\theta)$ in \eqref{eq:eig} is replaced by the posterior given the previous experiment designs and observed outcomes,
namely
\begin{equation}
p(\theta|\xi_{1:t-1},y_{1:t-1}) \propto p(\theta)\prod_{\tau=1}^{t-1}p(y_\tau|\theta,\xi_\tau).
\label{eq:iterated_prior}
\end{equation}
This now allows us to construct adaptive  experiments, wherein we use information gathered from previous iterations to select the designs used at future iterations.

\subsection{Estimating expected information gain}
Making even a single point estimate of \acrshort{EIG} when solving \eqref{eq:optim_problem} can be challenging because we must first estimate the unknown $p(y|\xi)$ or $p(\theta|y,\xi)$, and then take an expectation over $p(\theta)p(y|\theta,\xi)$. Nested Monte Carlo (NMC) estimators \citep{nmc}, which make a Monte Carlo approximation of both the inner and outer integrals, converge relatively slowly: at a rate $\mathcal{O}(T^{-1/3})$ in the total computational budget $T$.

\citet{foster2019variational} noted that this approach is inefficient because it makes a separate \acrlong{MC} approximation of the integrand for every sample of the outer integral. To share information between different samples, they proposed a number of variational estimators that used amortization, i.e.~they attempted to learn the functional form of the integrand rather than approximating it afresh each time. One of their approaches was based on amortized variational inference and required an \textit{inference network} $q_\phi(\theta|y)$ which takes as input $\phi, y$ and outputs a distribution over $\theta$. 
For any $q_\phi(\theta|y)$, we can construct a lower bound on $I(\xi)$. This is the \acrfull{BA}, or posterior, lower bound \citep{ba}
\begin{equation}
{I}_{BA}(\xi,\phi) \triangleq \E_{p(\theta)p(y|\theta,\xi)}[\log q_\phi(\theta|y)] + \entropy{p(\theta)},
\label{eq:ba}
\end{equation}
which was also used by \citep{pacheco2019variational} and which has found use representation learning \citep{poole2018variational} and maximizing information transmission over noisy channels \citep{ba}. 

To make high-quality approximations to $I(\xi)$, and simultaneously learn a good posterior approximation, \citet{foster2019variational} maximize this bound with respect to $\phi$. This approach is most effective when the bound is tight, i.e.~$\max_\phi I_{BA}(\xi, \phi) = I(\xi)$. For $I_{BA}(\xi, \phi)$, this occurs when it is possible to have $q_{\phi}(\theta|y) = p(y|\theta,\xi)$, i.e.~when the inference network is powerful enough to find the true posterior distribution for every $y$.

To obtain high-quality approximations of $I(\xi)$ even when the inference network cannot capture the true posterior, \citet{foster2019variational} also considered another variational estimator: variational nested Monte Carlo (VNMC).
This uses the inference network $q_\phi(\theta|y)$ in conjunction with additional samples to improve the estimate of the integrand. They showed that this leads to the following \textit{upper} bound on $I(\xi)$
\begin{equation}
I_{VNMC}(\xi, \phi, L) \triangleq \E\left[\log \frac{p(y|\theta_0,\xi)}{\frac{1}{L}\sum_{\ell=1}^L \frac{p(\theta_\ell)p(y|\theta_\ell,\xi)}{q_\phi(\theta_\ell|y)}}\right],
\label{eq:vnmc}
\end{equation}
where the expectation is over $p(\theta_0)p(y|\theta_0,\xi)q_\phi(\theta_{1:L}|y)$. The inference network in VNMC is trained by minimization, in the same way $I_{BA}$ is trained by maximization. $I_{VNMC}$ has the attractive feature that the bound becomes tight as $L\to\infty$, even if $q_\phi(\theta_\ell|y)$ is not powerful enough to directly represent the true posterior.

\subsection{Optimizing the EIG}

The experimental design problem is to find the design that maximizes the EIG. Therefore, as well as finding a way to estimate EIG, existing approaches subsequently need to find a way of searching across $\Xi$ to find promising designs. At a high-level, most existing approaches propose a two-stage procedure in which noisy estimates of $I(\xi)$ are made, and a separate optimization procedure selects the candidate design $\xi$ to evaluate next. 

\citet{kleinegesse2018efficient} and \citet{foster2019variational} both use \acrfull{BO} for this outer optimization step, a black-box optimization method that is tolerant to noise in the estimates of the objective function~\citep{snoek2012practical}, in this case $I(\xi)$. Some approaches \citep{watson2017quest+,lyu2019ultra} instead select a finite number of candidate designs in $\Xi$ and estimate $I(\xi)$ at each candidate, with some refining this process further by adaptively allocating computational resources between these designs~\citep{vincent2017,rainforth2017thesis}.
Another suggested approach is to use MCMC methods to carry out this outer optimization~\citep{amzal2006bayesian,muller2005simulation}.



\section{GRADIENT-BASED BOED}
\label{sec:method}

Our central proposal is to replace the two-stage procedure outlined above with a single stage that simultaneously estimates $I(\xi)$ and optimizes $\xi$. 
This has the critical advantage of allowing SGA to be directly applied to the design optimization.
Not only does this provide substantial computational gains over approaches which must construct separate estimates for each design considered, but it also provides the potential to scale to substantially higher dimensional design problems than those which can be effectively tackled with existing approaches.
Since we take gradients with respect to $\xi$, we henceforth assume that $\Xi$ is continuous.

In our approach, we utilize variational \textit{lower bounds} on $I$. Specifically, suppose we have a bound $\mathcal{L}(\xi, \phi) \le I(\xi)$ with variational parameters $\phi$. For fixed $\xi$, the estimate of $I(\xi)$ improves as we maximize with respect to $\phi$. We propose to maximize $\mathcal{L}$ \textit{jointly} with respect to $(\xi, \phi)$. As we train $\phi$, the variational approximation improves; as we train $\xi$ our design moves to regions where the lower bound on \acrshort{EIG} is largest. 
By tackling this as a single optimization problem over $(\xi, \phi)$, we obviate the need to have an outer optimizer for $\xi$. Using a lower bound is important because it allows us to perform a single maximization over $(\xi, \phi)$, rather than a more complex optimization such as the max-min optimization that would result if we used an upper bound.

In practice, we do not have lower bounds on $I$ that we can evaluate and differentiate in closed form. Instead, we have bounds that are expectations over $p(\theta)p(y|\theta,\xi)$. Fortunately, we can still maximize these lower bounds with respect to $(\xi, \phi)$ by using SGA, which is known to remain effective in high dimensions~\citep{bottou2010large}. 

\subsection{\acrfull{BA}}
We now make our first concrete proposal for the lower bound $\mathcal{L}(\xi, \phi)$: the BA bound $I_{BA}$, as defined in \eqref{eq:ba}. The difference is we will now optimize $(\xi, \phi)$ jointly whereas previously only $\phi$ was trained using gradients. 
To perform SGA, we use the following unbiased estimators for $\partial I_{BA}/\partial \phi$ and $\partial I_{BA}/\partial \xi$
\begin{align}
	\widehat{\frac{\partial I_{BA}}{\partial \phi}} &= \frac{1}{N}\sum_{n=1}^N \ppx{\phi} \log q_\phi(\theta_n|y_n),\\
	\widehat{\frac{\partial I_{BA}}{\partial \xi}} &= \frac{1}{N} \sum_{n=1}^{N} \log q_\phi(\theta_n|y_n) \ppx{\xi}\log p(y_n|\theta_n,\xi)
\end{align}
where $\theta_n,y_n \iid p(\theta)p(y|\theta,\xi)$.  The estimator of $\partial I_{BA}/\partial \xi$ is a score function estimator, 
other possibilities are discussed in Section~\ref{sec:gradients}.

\subsection{Adaptive contrastive estimation (ACE)}
The BA bound provides one specific case of our one-stage procedure for optimal experimental design. We now introduce a new lower bound that improves upon $I_{BA}$. The potential issue with the BA bound is that it may not be sufficiently tight,
which happens when the inference network cannot represent the true posterior. One possible solution is to introduce additional samples, as in the VNMC estimator \eqref{eq:vnmc}. However, we cannot use VNMC directly for a one-stage procedure: since it is an upper bound, we must minimize it with respect to $\phi$, but we still wish to maximize with respect to $\xi$. 

Looking more closely at the VNMC bound, we see that its main failure case is when the denominator strongly \textit{under-estimates} $p(y|\xi)$, which can happen when all the inner samples $\theta_1, ..., \theta_L$ miss regions where the joint $p(\theta_\ell)p(y|\theta_\ell,\xi)$ is large. In addition to the samples $\theta_{1:L}$, we also have the original sample $\theta_0$ from which $y$ was sampled. One way to avoid the under-estimation in the denominator would be to include this sample, giving
\begin{equation}
	{I}_{ACE}(\xi, \phi, L) = \E\left[\log \frac{p(y|\theta_0,\xi)}{\frac{1}{L+1}\sum_{\ell=0}^L \frac{p(\theta_\ell)p(y|\theta_\ell,\xi)}{q_\phi(\theta_\ell|y)}}\right]
	\label{eq:ace}
\end{equation}
where the expectation is with respect to $p(\theta_0)p(y|\theta_0,\xi)q(\theta_{1:L}|y)$. In fact, by including $\theta_0$ we cause the denominator to now \textit{over-estimate} $p(y|\xi)$ which results in a new \textbf{lower bound} on $I(\xi)$ which can be jointly maximized with respect to $(\xi, \phi)$.
The samples $\theta_{1:L}$ can now be seen as contrasts to the original sample $\theta_0$. For this reason, we call $\theta_{1:L}$ \textit{contrastive samples} and we call \eqref{eq:ace} the \textbf{adaptive contrastive estimate (ACE)} of EIG.
The following theorem establishes that ${I}_{ACE}$ is a valid lower bound on the EIG which becomes tight as $L\to \infty$.
\begin{restatable}{theorem}{lemace}
\label{lemma:ace}
For any model $p(\theta)p(y|\theta,\xi)$ and inference network $q_\phi(\theta|y)$, we have the following:\vspace{-5pt}
\begin{enumerate}
	\item $I_{ACE}$ is a lower bound on $I(\xi)$ and we can characterize the error term as an expected KL divergence:
	\begin{align*}
	I&(\xi) - {I}_{ACE}(\xi,\phi,L) \\
	 &= \E_{p(y|\xi)}\left[ KL\left( P(\theta_{0:L}|y) \middle|\middle| \prod_\ell q_\phi(\theta_\ell|y) \right) \right] \ge 0, \\
	P&(\theta_{0:L}|y) = \frac{1}{L+1}\sum_{\ell=0}^L p(\theta_\ell|y,\xi)\prod_{k\ne\ell} q_\phi(\theta_k|y).
	\end{align*} 
	\item As $L\to\infty$, we recover the true \acrshort{EIG}:\newline
	$\lim_{L\to \infty}{I}_{ACE}(\xi,\phi,L) = I(\xi)$.
	\item The ACE bound is monotonically increasing in $L$: ${I}_{ACE}(\xi, \phi, L_2) \ge {I}_{ACE}(\xi,\phi,L_1)$ for $L_2 \ge L_1 \ge 0$.
	\item If the inference network equals the true posterior $q_\phi(\theta|y) = p(\theta|y,\xi)$, then ${I}_{ACE}(\xi,\phi,L) = I(\xi), \forall L$.
	  \end{enumerate}
\end{restatable}
See Appendix~\ref{app:method} for the proof and additional results.
Gradient estimation for ACE is discussed in Section~\ref{sec:gradients}. 
We note that, to the best of our knowledge, $I_{ACE}$ has not previously appeared in the \acrshort{OED} literature.\footnote{Aside from a recent blog post~\citep{sobolev2019thoughts} we believe this bound has not previously been suggested in any context.}

\subsection{Prior contrastive estimation (PCE)}
Theorem~\ref{lemma:ace} tells us that $I_{ACE}$ can become close to $I(\xi)$ if either: 1) the inference network becomes close to the true posterior $p(\theta|y,\xi)$, 2) we increase the number of contrastive samples $L$. The BA bound only becomes tight in case 1). A special case of ACE is to replace the inference network $q_\phi(\theta|y)$ with a fixed distribution and rely on the contrastive samples to make good estimates of $I(\xi)$, only becoming tight in case 2), i.e.~as $L\to\infty$. This simplification can speed up training, since we no longer need to learn additional parameters $\phi$.

To explore this, we propose the \textbf{prior contrastive estimation (PCE)} bound, in which the prior $p(\theta)$ is used to generate contrastive samples:
\begin{equation}
{I}_{PCE}(\xi, L) \triangleq \E \left[\log \frac{p(y|\theta_0,\xi)}{\frac{1}{L+1}\sum_{\ell=0}^{L} p(y|\theta_\ell,\xi)}\right],
\label{eq:nce_likelihood}
\end{equation}
where the expectation is over $p(\theta_0)p(y|\theta_0,\xi)p(\theta_{1:L})$.
Whilst inherently less powerful than ACE, PCE can be effective when the prior and posterior are similar, such that $p(\theta)$ is a suitable proposal to estimate $p(y|\xi)$. 

Though, to the best of our knowledge, this bound has not been applied to \acrshort{OED} before, we note that it shares a connection to the information noise contrastive estimation (InfoNCE) bound on mutual information used in representation learning \citep{oord2018representation}. Given $K$ data samples $x_k$, corresponding representations $z_k$, and a critic $f_\psi(x, z) \ge 0$, we have
\begin{equation}
	\text{MI}(x; z) \ge \E\left[\frac{1}{K}\sum_{k=1}^K \log\frac{f_\psi(x_k, z_k)}{\frac{1}{K}\sum_{\ell=1}^K f_\psi(x_\ell, z_k)}	\right]
\end{equation}
where the expectation is over $p(x)p(z|x)$, $p(x)$ is the data distribution, and $p(z|x)$ is the encoder. 
\citet{poole2018variational} showed that the encoder density $p(z|x)$ is the optimal critic, although it is rarely known in closed form in the representation learning context. Writing $\theta$ for $x$ and $y$ for $z$, we note the mathematical connection between this optimal case and $I_{PCE}$. 

\subsection{Likelihood-free ACE}
In some models such as random effects models, the likelihood $p(y|\theta,\xi)$ is not known in closed form but can be sampled from. This presents a problem when computing $I_{ACE}$ or its derivatives because the likelihood appears in \eqref{eq:ace}. To allow ACE to be used for these kinds of models, we now show that using a unnormalized approximation to the likelihood still results in a valid lower bound on the \acrshort{EIG}. In fact, if using a parametrized likelihood approximation $f_\psi$, it is then possible to train $\psi$ jointly with $(\xi, \phi)$ to approximate the likelihood, learn an inference network, and find the optimal design through the solution to a single optimization problem. The following theorem, whose proof is presented in Appendix~\ref{app:method}, shows that replacing the likelihood with an unnormalized approximation does result in a valid lower bound on \acrshort{EIG}.
\begin{restatable}{theorem}{lemlf}
\label{lemma:lf}
Consider a model $p(\theta)p(y|\theta,\xi)$ and inference network $q_\phi(\theta|y)$. Let $f_\psi(\theta,y) \ge 0$ be an unnormalized likelihood approximation. Then,
\begin{equation}
	I(\xi) \ge \E\left[\log \frac{f_\psi(\theta_0,y)}{\frac{1}{L+1}\sum_{\ell=0}^L \frac{p(\theta_\ell)f_\psi(\theta_\ell,y)}{q_\phi(\theta_\ell|y)}} \right]
	\label{eq:ace_lf}
\end{equation}
where the expectation is over $p(\theta_0)p(y|\theta_0,\xi)q_\phi(\theta_{1:L}|y)$.
\end{restatable}

\subsection{Iterated experimental design with ACE}
In iterated experimental design, we replace $p(\theta)$ by $p(\theta|y_{1:t-1},\xi_{1:t-1})$ as per \eqref{eq:iterated_prior}. We can sample $p(\theta|y_{1:t-1},\xi_{1:t-1})$ by performing inference. Whilst variational inference also provides a closed form estimate of the posterior density, some other inference methods do not. This is problematic because the prior density appears in \eqref{eq:ace}. Fortunately, it is sufficient to know the density \textit{up to proportionality} \citep{foster2019variational}. Indeed if $p(\theta) = A \cdot \gamma(\theta)$ where $A$ does not depend on $(\xi, \phi, y)$ and $\gamma$ is an unnormalized density, then 
\begin{equation}
	{I}(\xi) \ge \E\left[\log \frac{p(y|\theta_0,\xi)}{\frac{1}{L+1}\sum_{\ell=0}^L \frac{\gamma(\theta_\ell)p(y|\theta_\ell,\xi)}{q_\phi(\theta_\ell|y)}}\right] - \log A
\end{equation}
and the derivatives of $\log A$ are simply zero.

\subsection{Gradient estimation for ACE}
\label{sec:gradients}
To optimize the ACE bound with respect to $(\xi, \phi)$ we need unbiased gradient estimators of $\partial I_{ACE}/\partial \xi$ and $\partial I_{ACE}/\partial \phi$.
The simplest form of the $\xi$-gradient is
\begin{equation}
	\pypx{{I}_{ACE}}{\xi} = \E\left[ \pypx{g}{\xi} + g \cdot \ppx{\xi}\log p(y|\theta_0,\xi) \right]
	\label{eq:reinforce}
\end{equation}
where the expectation is with respect to $p(\theta_0)p(y|\theta,\xi)q(\theta_{1:L}|y)$, and
\begin{equation}
	g(y, \theta_{0:L}, \phi, \xi) = \log \frac{p(y|\theta_0,\xi)}{\frac{1}{L+1}\sum_{\ell=0}^L \frac{p(\theta_\ell)p(y|\theta_\ell,\xi)}{q_\phi(\theta_\ell|y)}}.
\end{equation}
Estimating the expectation \eqref{eq:reinforce} directly using Monte Carlo gives the score function, or REINFORCE, estimator. 
Unfortunately, this is often high variance, and reducing gradient estimator variance is often important in solving challenging experimental design problems.

One variance reduction method is reparameterization. 
For this, we introduce random variables $\epsilon, \epsilon'_{1:L}$ which do not depend on $(\xi, \phi)$ along with representations of $y$ and $\theta$ as deterministic functions of these variables: $y = y(\theta_0, \xi, \epsilon)$ and $\theta_\ell = \theta(y, \phi, \epsilon'_\ell)$. This now permits the reparameterized gradient
\begin{equation}
	\pypx{{I}_{ACE}}{\xi} = \E\left[ \pypx{g}{\xi} + \pypx{g}{y}\pypx{y}{\xi} + \sum_{\ell=1}^L \pypx{g}{\theta_\ell}\pypx{\theta_\ell}{y}\pypx{y}{\xi} \right]
\end{equation}
where the expectation is over $p(\theta_0)p(\epsilon)p(\epsilon'_{1:L})$. A Monte Carlo approximation of this expectation is typically a much lower variance estimator for the true $\xi$-gradient.

Alternatively, if $y$ is a discrete random variable we can sum over the possible values $\mathcal{Y}$. This approach is known as Rao-Blackwellization and gives
\begin{equation}
	\pypx{{I}_{ACE}}{\xi} = \sum_{y\in\mathcal{Y}}\E\left[ \pypx{g}{\xi}\,  p(y|\theta_0,\xi) + g \, \ppx{\xi}p(y|\theta_0,\xi) \right]
	\label{eq:raoblackwellise}
\end{equation}
where the expectation is now over $p(\theta_0)\prod_{\ell=1}^L q_\phi(\theta_\ell|y)$.

Turning to $\partial I_{ACE}/\partial \phi$, we note that if $\theta_{1:L}$ are reparameterizable (i.e. can be expressed $\theta_\ell = \theta(y, \phi, \epsilon'_\ell)$), then we can utilize the double reparameterization of \citet{tucker2018doubly}; for full details see Appendix~\ref{app:drep}.



\section{EXPERIMENTS}
\label{sec:experiments}

\begin{figure}[t]
	\centering
	\vspace{-8pt}
	\begin{subfigure}[b]{0.30\textwidth}
		\centering
		\includegraphics[width=\textwidth]{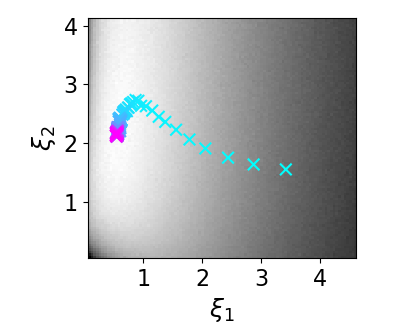}
	\end{subfigure}
	\vspace{-8pt}
    \caption{A sample trajectory for the death process. The grayscale shows the EIG surface (white is maximal), whilst crosses show the optimization trajectory of $\xi$ using ACE with pink representing later steps. See Sec.~\ref{sec:deathprocess} for details.
    \vspace{-10pt}}
	\label{fig:death_trajectory}
\end{figure}

We now learn optimal experimental designs in five scenarios: the \textbf{death process}, a well known two-dimensional design problem from epidemiology; a non-conjugate \textbf{regression} model with a 400-dimensional design; an ablation study in the setting of \textbf{advertising}; a real-world \textbf{biomolecular docking} problem from pharmacology  in 100 dimensions; and a \textbf{constant elasticity of substitution} iterated design problem in behavioural economics with 6 dimensional designs.

\subsection{Evaluating experimental designs}
We first discuss which metrics we will use to judge the quality of the designs we obtain. Our primary metric on designs is, of course, the \acrshort{EIG}. We prefer designs with high \acrshort{EIG}s. In some cases, we can evaluate the \acrshort{EIG} analytically. In other cases, we can use a sufficiently large number of samples in a \acrshort{NMC}~\citep{nmc} estimator to be sure that we have estimates that are sufficiently accurate to compare designs.

To explore the limits of our methods, we will also consider scenarios where neither of these approaches is suitable. In these cases, we pair the ACE lower bound (with $\xi$ fixed for evaluation) with the VNMC upper bound \citep{foster2019variational} to trap the true \acrshort{EIG} value---if the lower bound of one design is higher than the upper bound for another, we can be sure that the first design is superior (noting that the bounds themselves can be tractably estimated to a very high accuracy).

In some settings, when we know the true optimal design $\xi^*$, we will also consider the \textit{design error} $\|\xi^* - \xi\|$, i.e. how close our design is to the optimal design.

In iterated experiment design, as well as designing experiments, we must also perform inference on the latent variable $\theta$ after each iteration. Here, we also investigate the quality of the final posterior. Specifically, if $p(\theta|y_{1:t},\xi_{1:t})$ is the posterior after $t$ experiments, 
we use the posterior entropy, and the posterior RMSE $\E_{\theta \sim p(\theta|y_{1:t},\xi_{1:t})}\left[ (\theta - \theta^*)^2 \right]^{1/2}$.
We prefer low entropies and low RMSE values.

\subsection{Death process}
\label{sec:deathprocess}

\begin{figure}[t]
	\centering
	\vspace{-8pt}
	\begin{subfigure}[b]{0.34\textwidth}
		\centering
		\includegraphics[width=\textwidth]{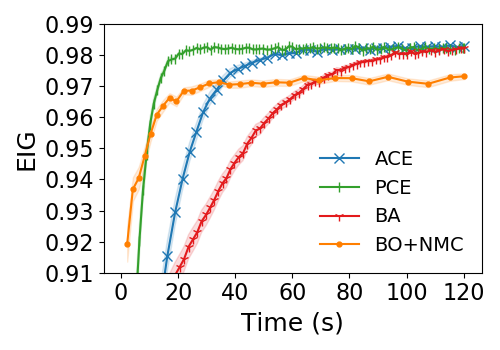}
	\end{subfigure}
	\vspace{-6pt}
	\caption{Optimization of EIG for the death process as a function of wall clock time. 
    We depict the mean and $\pm1$ standard error (s.e.) from 100 runs. The final EIG values (rightmost points) are as follows: [ACE] $\mathbf{0.9830 \pm 0.0001}$, [PCE] $0.9822\pm 0.0001$, [BA] $0.9822 \pm 0.0002$, [BO] $0.9732 \pm 0.0009$. See Sec.~\ref{sec:deathprocess} for details.}
	\vspace{-10pt}
	\label{fig:death_rb_bo}
\end{figure}

We consider an example from epidemiology, the death process \citep{cook2008optimal,kleinegesse2018efficient}, in which a population of $N=10$ individuals transitions from healthy to infected states at a constant but unknown rate $\theta$. We can measure the number of infected individuals at two different times $\xi_1$ and $\xi_1 + \xi_2$ where $\xi_1,\xi_2 \ge0$. Our aim is to infer the infection rate $\theta$ from these observations.
For full details of the prior and likelihood used, see Appendix~\ref{app:death}. 

On this problem, we apply gradient methods with Rao-Blackwellization over the 66 possible outcomes. 
Figure~\ref{fig:death_trajectory} shows a sample optimization trajectory with the approximate EIG surface. We compare against \acrshort{BO} using the Rao-Blackwellized NMC estimator of \citet{vincent2017}. Figure~\ref{fig:death_rb_bo} shows that, for the allowed time budget, all gradient methods perform better than BO even on this two-dimensional problem.

\subsection{Regression}
\label{sec:regression}

We now compare our one-stage gradient approaches to experimental design against a two-stage baseline on a high-dimensional design problem. We choose a general purpose Bayesian linear regression model with $n$ observations and $p$ features. The design $\xi$ is an $n \times p$ matrix; the latent variables are $\theta = (\mathbf{w}, \sigma)$, where $\mathbf{w}$ is the $p$ dimensional regression coefficient and $\sigma^2$ is the scalar variance. The $n$ outcomes are generated using a Normal likelihood $y_i \sim N(\bm{\xi}_i \cdot \mathbf{w}, \sigma)$ for $i=1,...,n$. Here $\bm{\xi}_i$ is the $i$th row of $\xi$. To avoid trivial solutions, we enforce the constraint $\|\bm{\xi}_i\|_1=1$ for all $i$. We use independent priors $w_j \sim \text{Laplace}(1)$ for $j=1,...,p$ and $\sigma \sim \text{Exp}(1)$. See 
Appendix~\ref{app:regression} for complete details.

We set $n=p=20$ and applied five methods to this 400 dimensional design problem: BA, ACE and PCE, as well as the VNMC estimator of \citet{foster2019variational}, with both \acrshort{BO} and random search to optimize over $\Xi$. The results are presented in Table~\ref{tab:regression_eig}. We note that the gradient methods strongly outperform the gradient-free baselines, with about double the final EIG. 

\begin{table}[t]
	\vspace{-6pt}
	\begin{center}
        \caption{Regression results. We estimate lower and upper bounds on the final EIG and present the mean and $\pm1$ s.e. from 10 runs. See Sec.~\ref{sec:regression} for details.  \vspace{-5pt}}
		\label{tab:regression_eig}
		{\renewcommand{\arraystretch}{1}
			\setlength\tabcolsep{4pt} 
			\begin{tabular}{lrr}
			\hline

				 \small Method & \small EIG l.b. & \small EIG u.b.  \\
				\hline
			    \small ACE & \small $16.1\pm 0.1 $ & \small $20.7 \pm 0.2$ \\
				 \small PCE & \small $16.6\pm 0.1$ & \small $21.5\pm 0.2$ \\
				 \small BA & \small $16.4 \pm 0.2$ & \small $21.1\pm 0.2$ \\
				\hline
				\small BO + VNMC & \small $7.3 \pm 0.1$ & \small $9.6 \pm 0.1$ \\
				\small Random Search + VNMC & \small $7.1 \pm 0.1$ & \small $9.4 \pm 0.1$ \\
				\hline
			\end{tabular}}
		\end{center}
		\vspace{-6pt}
	\end{table}

\subsection{Advertising}
\label{sec:largedim}

We now conduct a detailed ablation study on the effects of dimension on the quality of experimental designs produced using our gradient approaches and \acrshort{BO}. To further isolate the distinction between one-stage and two-stage approaches to \acrshort{OED}, we choose a setting in which we can compute $I(\xi)$ analytically. We give \acrshort{BO}, but not the gradient methods, access to a EIG oracle when making point evaluations of $I(\xi)$, i.e.~our two-stage baseline is spared the need to estimate $I(\xi)$. 
Thus we put \acrshort{BO} in the best possible position and ensure any gains are due to improvements from using gradient-based optimization.

Suppose that we are given an advertising budget of $B$ dollars that we need to allocate
among $D$ regions, i.e.~we choose $\bm{\xi} \ge \bm{0} $ with $\sum_{i=1}^D \xi_i=B$.
After conducting an ad campaign, we observe a vector of sales $\bm{y}$. We use this data to make inferences about the underlying market opportunities $\bm{\theta}$ in each region. Our prior incorporates the knowledge that neighbouring regions are more correlated than distant ones---this leads to an interesting experimental design problem because information can be pooled between regions. We can also compute the true EIG and optimal design $\xi^*$ analytically. For full details, see Appendix \ref{app:largedim}.

We compare the performance of four estimation and optimization methods on this problem, see Fig.~\ref{fig:varydim} for the results.  
The three gradient-based methods (ACE, PCE, BA) perform best, with the BO baseline struggling in dimensions $D \ge 6$, 
even though the latter has access to an EIG oracle.
PCE performed well in low dimensions, but degraded as the dimension increases and sampling from the prior becomes increasingly inefficient,
ACE and BA avoid this by learning adaptive proposal distributions.
 We note that since in this case the family of variational distributions used in ACE and BA include the true posterior, both methods yield similar performance. 

\begin{figure}[t]
	
	\centering
	\begin{subfigure}[b]{0.42\textwidth}
		\centering
		\vspace{-4pt}
		\includegraphics[width=\textwidth]{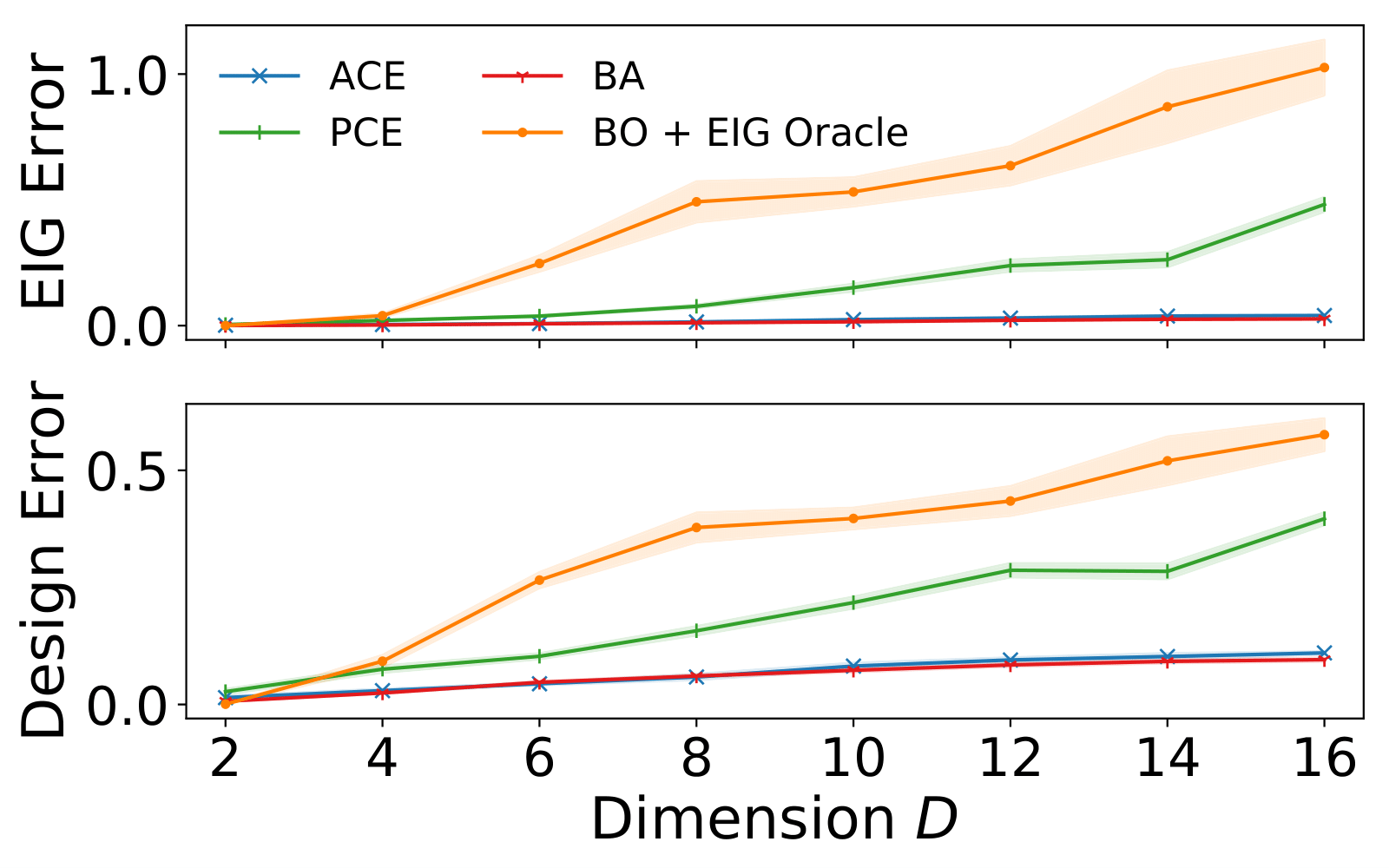}
	\end{subfigure}
	\vspace{-1pt}
	\caption{Mean absolute EIG and design errors for the marketing model in Sec.~\ref{sec:largedim} 
		averaged over 10 runs. The EIG is normalized such that an EIG error of unity corresponds to doing no better than a uniform budget, i.e. $\xi_i = B/D$ for $i=1,...,D$.
	\vspace{-10pt}}
	\label{fig:varydim}
\end{figure}

\subsection{Biomolecular docking}

\begin{figure*}
	\centering
	\vspace{-6pt}
	\begin{subfigure}[b]{0.76\textwidth}
		\centering
		\includegraphics[width=\textwidth]{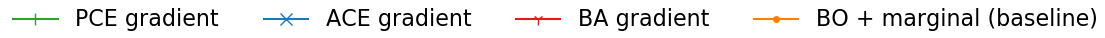}
	\end{subfigure}
	\hfill
	\begin{subfigure}[b]{0.24\textwidth}
		\centering
		\includegraphics[width=\textwidth]{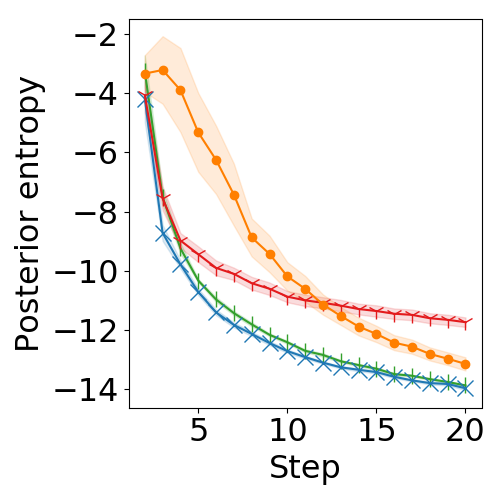}
		\caption{Total Entropy\label{fig:cesentropy}}
	\end{subfigure}
	\begin{subfigure}[b]{0.24\textwidth}
		\centering
		\includegraphics[width=\textwidth]{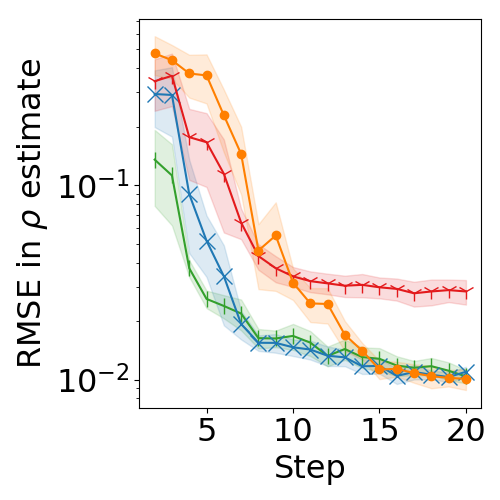}
		\caption{Posterior RMSE of $\rho$\label{fig:cesrho}}
	\end{subfigure}
	\begin{subfigure}[b]{0.24\textwidth}
		\centering
		\includegraphics[width=\textwidth]{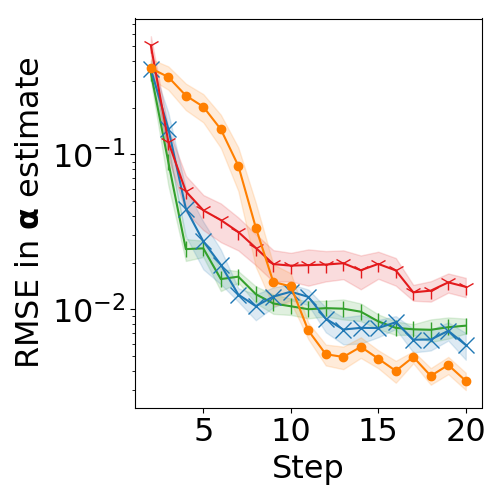}
		\caption{Posterior RMSE of $\bm{\alpha}$ \label{fig:cesalpha}}
	\end{subfigure}
	\begin{subfigure}[b]{0.24\textwidth}
		\centering
		\includegraphics[width=\textwidth]{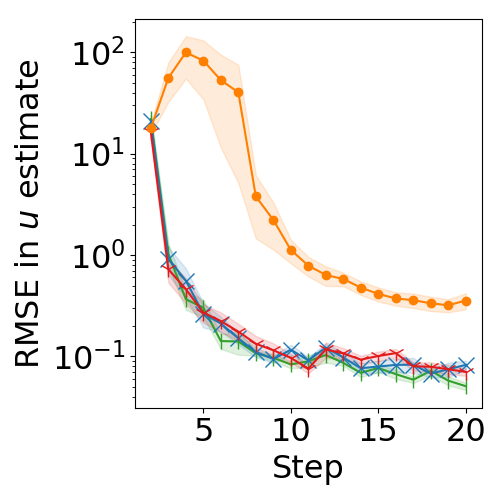}
		\caption{Posterior RMSE of $u$ \label{fig:cesu}}
	\end{subfigure}
	\vspace{-4pt}
    \caption{Improvement in the posterior in the sequential CES experiment. Each step took 120 seconds for each method. 
    	We present the mean and $\pm1$ standard error from 10 runs. See Sec.~\ref{sec:ces} for details.
    \vspace{-10pt}}
	\label{fig:ces}
\end{figure*}

We now consider an experimental design problem of interest to the pharmacology community. Having demonstrated that our one-stage gradient methods compare favourably with two stage approaches, we now compare against designs crafted by domain experts.

In molecular docking, computational techniques are used to predict the binding affinity between a compound and a receptor. When synthesized in the lab, the two may bind---this is called a \textit{hit}. Learning a well-calibrated hit-rate model can guide how many compounds to evaluate for additional objectives, such as drug-likeness or toxicity, before experimental testing.
\citet{lyu2019ultra} modelled the probability of outcome $y_i$ being a hit, given the predicted binding affinity, or {docking score}  $\xi_i \in [-75,0]$, as
\begin{equation}
	p(y_i=1|\theta,\xi) = \text{bottom} + \frac{\text{top} - \text{bottom}}{1 + e^{-(\xi_i - \text{ee50}) \times \text{slope}}}
\end{equation}
where $\theta=(\text{top},\,\text{bottom},\,\text{ee50},\,\text{slope})$ with priors given in Appendix~\ref{app:docking}.

Of 150 million compounds, \citet{lyu2019ultra} selected a batch of compounds to experimentally test to best fit the sigmoid hit-rate model. They considered 6 candidate designs and selected one that maximized the EIG estimated by NMC. Here, we instead apply gradient-based BOED to search across candidate designs which consist of 100 docking scores $\xi_1, ..., \xi_{100}$. To evaluate our final designs, we present upper and lower bounds on the final EIG: see Table~\ref{tab:docking_eig}. We see that all gradient methods are able to outperform experts in terms of \acrshort{EIG}, and that ACE appears the best of the gradient methods. Figure~\ref{fig:docking_histogram} shows our designs are qualitatively different to those produced by experts.

\begin{table}[t]
	\begin{center}
		\caption{Biomolecular docking results showing 
			the mean and $\pm1$ s.e. from 10 runs. For the expert, we took the best design of \citet{lyu2019ultra} appropriately rescaled to consist of 100 docking scores for comparison.  \vspace{-5pt}}
		\label{tab:docking_eig}
		{\renewcommand{\arraystretch}{1}
			\setlength\tabcolsep{4pt} 
			\begin{tabular}{lll}
			\hline
				 \small Method & \small EIG lower bound & \small EIG upper bound  \\
				\hline
			    \small \textbf{ACE} & \small $\mathbf{1.0835\pm 0.0003} $ & \small $\mathbf{1.0852 \pm 0.0001}$ \\
				 \small PCE & \small $1.0825\pm 0.0002$ & \small $1.0839\pm 0.0002$ \\
				 \small BA & \small $1.0780 \pm 0.0003$ & \small $1.0794\pm 0.0003$ \\
				\hline
				\small Expert & \small $1.0191$ & \small $1.0227$ \\
				\hline
			\end{tabular}}
		\end{center}
		\vspace{-6pt}
	\end{table}

\vspace{-3pt}

\subsection{Constant elasticity of substitution}
\label{sec:ces}

We finally turn to \textit{iterated} experimental design in which we produce designs, generate data and make inference repeatedly. This problem therefore captures the end-to-end-process of experimentation and inference.

We consider an experiment in behavioural economics that was previously also considered by \citet{foster2019variational}. In this experiment, a participant is asked to compare baskets $\mathbf{x}, \mathbf{x}'$ of goods. The model assumes that their response (on a slider) is based on the difference in utility of the baskets, and the constant elasticity of substitution (CES) model \citep{arrow1961capital} governed by latent variables $(\rho,\bm{\alpha},u)$ is then used for this utility. The aim is to learn $(\rho,\bm{\alpha},u)$ characterizing the participant's utility. In the experiment, there are 20 sequential steps of experimentation with the same participant. We compare our gradient-based approach against the most successful approach of \cite{foster2019variational} that approximates the marginal density to form an upper bound on EIG, and BO to optimize $\xi$. For full details, see Appendix~\ref{app:ces}. 

Figure~\ref{fig:ces} shows that gradient-based methods are effective on this problem; both ACE and PCE decrease the posterior entropy and RMSEs on the latent variables faster and further than the baseline, whereas BA does not do so well. We suggest that the similar performance of ACE and PCE is due to the smaller changes in the posterior at middle and late steps, after much data has been accumulated: when the posterior does not change much at each step, $p(\theta|y_{1:t-1},\xi_{1:t-1})$ forms an effective proposal for estimating $p(y_t|\xi_t)$.

\begin{figure}[t]
	\centering
	\begin{subfigure}[b]{0.35\textwidth}
		\centering
		\includegraphics[width=\textwidth]{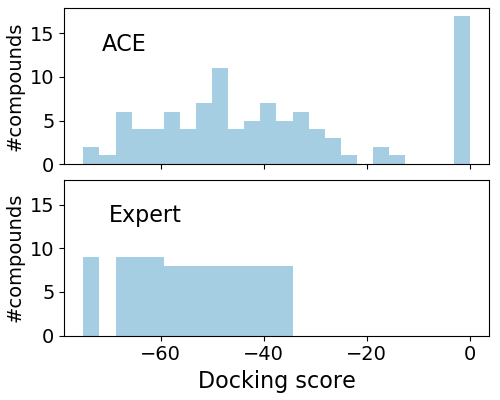}
	\end{subfigure}
	\vspace{-6pt}
	\caption{Designs for the biomolecular docking problem obtained by ACE and by \citet{lyu2019ultra}. Designs consist of 100 docking scores at which to test compounds.}
	\label{fig:docking_histogram}
	\vspace{-6pt}
\end{figure}



\vspace{-4pt}
\section{CONCLUSIONS}
\label{sec:discussion}
\vspace{-4pt}
We have introduced a new approach for Bayesian experimental design that does away with the two stages of estimating EIG and separately optimizing over $\Xi$. We use stochastic gradients to maximize a lower bound on $I(\xi)$ and so find optimal designs by solving a single optimization problem. This unification leads to substantially improved performance, especially on high-dimensional design problems.

Of the three lower bounds, $I_{BA}, I_{ACE}$ and $I_{PCE}$, we note that in all five experiments ACE generally did as well as the better of BA and PCE: we therefore recommend it as the default choice. BA performed well when the inference network could closely approximate the true posterior; PCE performed well when the prior was an adequate proposal for estimating $p(y|\xi)$ and does not require the training of variational parameters.


\clearpage
\section*{Acknowledgements}

AF gratefully acknowledges funding from EPSRC grant no. EP/N509711/1.
YWT's research leading to these results has received funding from the
European Research Council under the European Union's Seventh Framework
Programme (FP7/2007-2013) ERC grant agreement no. 617071.
TR gratefully acknowledges funding from Tencent AI Labs and a junior research fellowship supported by Christ Church, Oxford.

\bibliography{references}
\bibliographystyle{plain}

\clearpage
\appendix

\section{GRADIENT-BASED BOED}
\label{app:method}

We begin with the proof of Theorem~\ref{lemma:ace}, which we restate for convenience.

\lemace*

We add the further technical assumption that $p(\theta)p(y|\theta,\xi)/q_\phi(\theta|y)$ is bounded.

\begin{proof}

To begin with 1., we have the error term $\delta = I(\xi) - I_{ACE}(\xi,\phi,L)$ which can be written
	\begin{align}
	\delta &= \E\left[ \log \frac{\frac{1}{L+1}\sum_{\ell=0}^L \frac{p(\theta_\ell)p(y|\theta_\ell,\xi)}{q_\phi(\theta_\ell|y)}}{p(y|\xi)} \right] \\
	&= \E\left[ \log \frac{\frac{1}{L+1}\sum_{\ell=0}^L p(\theta_\ell|y)\prod_{k\ne\ell}q_\phi(\theta_k|y)}{\prod_{\ell=0}^L q_\phi(\theta_\ell|y)} \right] \\
	&= \E\left[ \log \frac{P(\theta_{0:\ell}|y)}{\prod_{\ell=0}^L q_\phi(\theta_\ell|y)} \right]
	\end{align}
	where the expectation is over $p(y|\xi)p(\theta_0|y,\xi)\prod_{\ell=1}^L q_\phi(\theta_\ell|y)$.
	Note that the integrand is symmetric under a permutation of the labels $0, ..., L$, so its expectation will be the same over the distribution $p(y|\xi)p(\theta_\ell|y,\xi)\prod_{k \ne \ell} q_\phi(\theta_k|y)$. Since $P(\theta_{0:L})$ is a mixture of distributions of this form, this then implies that the expectation will be the same if it is taken over the distribution $p(y|\xi)P(\theta_{0:L})$, yielding
	\begin{equation}
		\delta = \E_{p(y|\xi)P(\theta_{0:L}|y)}\left[ \log \frac{P(\theta_{0:L}|y)}{\prod_{\ell=0}^L q_\phi(\theta_\ell|y)} \right]
	\end{equation}
	which is the expected KL divergence required. We therefore have $\delta \ge 0$.	

	For 2., we use that $p(\theta)p(y|\theta,\xi)/q_\phi(\theta|y)$ is bounded. 
	The ACE denominator is a consistent estimator of the marginal likelihood. Indeed, 
	\begin{equation}
		\frac{1}{L+1} \frac{p(\theta_0)p(y|\theta_0,\xi)}{q_\phi(\theta_0|y)} \to 0
	\end{equation}
	and
	\begin{equation}
		\frac{1}{L+1} \sum_{\ell=1}^L \frac{p(\theta_\ell)p(y|\theta_\ell,\xi)}{q_\phi(\theta_\ell|y)} \to p(y|\xi)\  a.s.
	\end{equation}
	as $L\to \infty$ by the Strong Law of Large Numbers, since
	\begin{equation}
	\E_{q_\phi(\theta|y)}\left[\frac{p(\theta)p(y|\theta,\xi)}{q_\phi(\theta|y)} \right] = p(y|\xi).
	\end{equation}
	This establishes the a.s. pointwise convergence of the ACE integrand to $\log p(y|\theta_0,\xi)/p(y|\xi)$. Hence by Bounded Convergence Theorem, 
	\begin{equation}
	\hat{I}_{ACE}(\xi,\phi, L) \to I(\xi)
	\end{equation}
	as $L\to\infty$. 
	
	To establish 3., we use a similar approach to 1. We let $\varepsilon = I_{ACE}(\xi,\phi,L_2) - I_{ACE}(\xi,\phi,L_1)$. Then
	\begin{align}
	\varepsilon &= \E\left[ \log \frac{\frac{1}{L_1+1}\sum_{\ell=0}^{L_1} \frac{p(\theta_\ell)p(y|\theta_\ell,\xi)}{q(\theta_\ell|y)}}{\frac{1}{L_2+1} \sum_{\ell=0}^{L_2} \frac{p(\theta_\ell)p(y|\theta_\ell,\xi)}{q(\theta_\ell|y)}} \right] \\
	&= \E\left[\log \frac{Q(\theta_{0:L_2}|y)}{\frac{1}{L_2+1}\sum_{\ell=0}^{L_2}p(\theta_\ell|y)\prod_{k\ne\ell}q(\theta_k|y)} \right]	
	\end{align}
	where the expectation is over $p(y|\xi)p(\theta_0|y,\xi)\prod_{\ell=1}^{L_2} q(\theta_\ell|y)$ and 
	\begin{equation}
	Q(\theta_{0:L_2}|y) = \frac{1}{L_1+1}\sum_{\ell=0}^{L_1}p(\theta_\ell|y)\prod_{k\ne\ell}^{L_2}q(\theta_k|y).
	\end{equation}
	As in 1., the integrand is unchanged if we permute the labels $0, ..., {L_1}$. By this symmetry, the expectation is the same when taken over the distribution $p(y|\xi)Q(\theta_{0:{L_2}}|y)$. We therefore recognise $\varepsilon$ as the expectation of a KL divergence. Hence $\varepsilon \ge 0$ as required.
	
	4. follows by Bayes Theorem, i.e.
	\begin{equation}
	\frac{p(\theta)p(y|\theta,\xi)}{p(\theta|y,\xi)} = p(y|\xi).
	\end{equation}
	which completes the proof.
\end{proof}

We also present the proof of Theorem~\ref{lemma:lf}.

\lemlf*

\begin{proof}
Initially, we note that the contrastive samples $\theta_1, ..., \theta_L$ do not carry additional information about $\theta_0$. Formally, we consider the mutual information between $\theta_0$ and the random variable $(y,\theta_1, ..., \theta_L)$. Using the Chain Rule for mutual information we have
\begin{equation}
\begin{split}
&\text{MI}(\theta_0;(y,\theta_1,...,\theta_L)) \\ &\ = \text{MI}(\theta_0;y) + \text{MI}(\theta_0;(\theta_1,...,\theta_L)|y)
\end{split}
\end{equation}
Now $\text{MI}(\theta_0;(\theta_1,...,\theta_L)|y)=0$ since $\theta_\ell$ ($\ell>0$) are conditionally independent of $\theta_0$ given $y$. Therefore
\begin{equation}
\text{MI}(\theta_0;(y,\theta_1,...,\theta_L)) = \text{MI}(\theta_0;y) = I(\xi).
\end{equation}
We now use the Donsker-Varadhan representation of mutual information \citep{donsker1975asymptotic}. Specifically, for random variables $A, B$ with joint distribution $p(a,b)$ and any measurable function $T(a,b)$ we have
\begin{equation}
\begin{split}
	&\text{MI}(A;B) \\ &\ \ge \E_{p(a,b)}[T(a,b)] - \log\E_{p(a)p(b)}\left[e^{T(a,b)}\right].
\end{split}
	\label{eq:dv}
\end{equation}
We now use this representation with $a=\theta_0, b=(y,\theta_1, ..., \theta_L)$ and $T(a,b)$ the integrand 
\begin{equation}
T(\theta_0, (y, \theta_{1:L})) = \log \frac{f_\psi(\theta_0,y)}{\frac{1}{L+1}\sum_{\ell=0}^L\frac{p(\theta_\ell)f_\psi(\theta_\ell,y)}{q_\phi(\theta_\ell|y)}}.
\end{equation}
We compute the second term in \eqref{eq:dv}, $Z = \E_{p(a)p(b)}\left[e^{T(a,b)}\right]$.
\begin{align}
	Z&=\E_{p(\theta_0)p(y|\xi)q_\phi(\theta_{1:L}|y)}\left[ \frac{f_\psi(\theta_0,y)}{\frac{1}{L+1}\sum_{\ell=0}^L\frac{p(\theta_\ell)f_\psi(\theta_\ell,y)}{q_\phi(\theta_\ell|y)}} \right] \\
	&=\E_{p(y|\xi)q_\phi(\theta_{0:L}|y)}\left[ \frac{\frac{p(\theta_0)f_\psi(\theta_0,y)}{q_\phi(\theta_0|y)}}{\frac{1}{L+1}\sum_{\ell=0}^L\frac{p(\theta_\ell)f_\psi(\theta_\ell,y)}{q_\phi(\theta_\ell|y)}} \right]  \\
	&= \E_{p(y|\xi)q_\phi(\theta_{0:L}|y)}\left[ \frac{\frac{1}{L+1}\sum_{\ell=0}^L\frac{p(\theta_\ell)f_\psi(\theta_\ell,y)}{q_\phi(\theta_\ell|y)}}{\frac{1}{L+1}\sum_{\ell=0}^L\frac{p(\theta_\ell)f_\psi(\theta_\ell,y)}{q_\phi(\theta_\ell|y)}} \right] \\
	&= 1
\end{align}
where the second to last line follows by symmetry. This establishes that $\log Z=0$, and so \eqref{eq:ace_lf} constitutes a valid lower bound on $I(\xi)$. That is
\begin{align}
	I(\xi) &\ge \E\left[ \log \frac{f_\psi(y,\theta_0)}{\frac{1}{L+1}\sum_{\ell=0}^L \frac{p(\theta)f_\psi(y,\theta_\ell)}{q_\phi(\theta_\ell,y)}} \right]
	\label{eq:proportionality}
\end{align}
which completes the proof.
\end{proof}

The following theorem establishes a condition under which the maximum of the ACE objective converges to the maximum of the EIG as $L\to\infty$.

\begin{restatable}{theorem}{thmconv}
\label{thm:conv}
Consider a model $p(\theta)p(y|\theta,\xi)$ such that
\begin{equation}
	C \triangleq \sup_{\xi \in \Xi}\inf_{\phi\in\Phi} \E_{p(\theta)p(y|\theta,\xi)}\left[ \frac{p(\theta|y,\xi)}{q_\phi(\theta|y,\xi)} \right] < \infty.
\end{equation}
and $I^* \triangleq \sup_{\xi \in \Xi} I(\xi) < \infty$.
Let $q_\phi(\theta|y)$ be an inference network and let
\begin{equation}
	I_L = \sup_{\xi \in \Xi, \phi \in \Phi} I_{ACE}(\xi, \phi, L).
\end{equation}
Then,
\begin{equation}
	0 \le I^* - I_L \le \frac{C-1}{L+1}
\end{equation}
and in particular $I_L \to I^*$ as $L \to \infty$.
\end{restatable}
\begin{proof}
We have $0 \le I^* - I_L$ since $I_{ACE}$ is a lower bound on $I(\xi)$ by Theorem~\ref{lemma:ace}.

Next, we consider $\Delta(\xi,\phi,L) = I(\xi) - I_{ACE}(\xi, \phi, L)$. We have
\begin{equation}
\label{eq:difference}
	\Delta = \E_{p(\theta_0)p(y|\theta_0,\xi)q_\phi(\theta_{1:L}|y)}\left[ \log \frac{Y_L}{p(y|\xi)} \right]
\end{equation}
where
\begin{equation}
	Y_L =\frac{1}{L+1} \sum_{\ell=0}^L w_\ell \quad\text{and}\quad
	w_\ell = \frac{p(\theta_\ell)p(y|\theta_\ell,\xi)}{q_\phi(\theta_\ell|y)};
\end{equation}
we write \eqref{eq:difference} as
\begin{equation}
	\Delta = \E\left[\log \left(1 + \frac{Y_L - p(y|\xi)}{p(y|\xi)} \right)\right]
\end{equation}
and we apply the inequality $\log(1+x)\le x$ to give
\begin{equation}
	\Delta \le \E\left[ \frac{Y_L - p(y|\xi)}{p(y|\xi)} \right].
\end{equation}
We now observe that for $\ell > 0$, $\E_{q_\phi(\theta_\ell|y)}[w_\ell] = p(y|\xi)$ and hence, taking a partial expectation over $\theta_{1:L}$ we have
\begin{align}
	\Delta &\le \E_{p(\theta_0)p(y|\theta_0,\xi)}\left[ \frac{w_0-p(y|\xi)}{(L+1)p(y|\xi)} \right] \\
	& \le \frac{1}{L+1}\left(\E_{p(\theta_0)p(y|\theta_0,\xi)}\left[ \frac{p(\theta_0|y,\xi)}{q_\phi(\theta_0|y)}\right] - 1\right)
\end{align}
Hence
\begin{align}
	I^* - I_L &= \sup_{\xi \in \Xi}I(\xi) - \sup_{\xi\in\Xi,\phi\in\Phi} I_{ACE}(\xi,\phi,L) ] \\
	 &\le \sup_{\xi \in \Xi}[ I(\xi) - \sup_{\phi\in\Phi} I_{ACE}(\xi,\phi,L) ] \\
	& \le \sup_{\xi\in\Xi}\inf_{\phi\in\Phi}[\Delta(\xi,\phi,L)] \\
	& \le \frac{C-1}{L+1}
\end{align}
as required.
\end{proof}

\subsection{Double reparametrization}
\label{app:drep}
We have the $\phi$-gradient of the ACE objective
\begin{equation}
	\pypx{{I}_{ACE}}{\phi} = \E_{p(\theta_0)p(y|\theta_0,\xi)}\left[ \left . -\pypx{\mathcal{L}}{\phi}\right\rvert_{\theta_0,y} \right]
\end{equation}
where $\mathcal{L}$ is our estimate of the marginal likelihood with gradient
\begin{equation}
	\left . \pypx{\mathcal{L}}{\phi}\right\rvert_{\theta_0,y} = \ppx{\phi}\E_{q_\phi(\theta_{1:L}|y)}\left[\log \left. \left(\sum_{\ell=0}^{L} w_\ell \right) \right\rvert \theta_0,y \right]
\end{equation}
where
\begin{equation}
	w_\ell = \frac{p(\theta_\ell)p(y|\theta_\ell,\xi)}{q_\phi(\theta_\ell|y)}.
\end{equation}
If $q_\phi(\theta|y)$ is reparameterizable as a function of $\phi$, then we can apply \textit{double} reparameterization to this gradient. Indeed, were it not for the $w_0$ term, this would be exactly the IWAE of \citet{burda2015importance}. We exploit the double reparameterization of \citet{tucker2018doubly} with a minor variation to account for $w_0$ to obtain a low variance gradient estimator. 

The doubly reparametrized gradient for ACE takes the form
\begin{equation}
\pypx{{I}_{ACE}}{\phi} = \E_{p(\theta_0)p(y|\theta_0,\xi)q_\phi(\theta_{1:L}|y)}\left[ \sum_{\ell=0}^L v_\ell  \right]
\end{equation} 
where
\begin{equation}
v_0 = \frac{w_0}{\sum_{m=0}^{L}w_m}\ppx{\phi}\log q_\phi(\theta_0|y)
\end{equation}
and for $\ell > 0$
\begin{equation}
v_\ell = -\left(\frac{w_\ell}{\sum_{m=0}^{L}w_m} \right)^2 \pypx{\log w_\ell}{\theta_\ell}\pypx{\theta_\ell}{\phi}.
\end{equation}

\subsection{Alternative gradient}
We begin with an observation: the true integrand when computing the EIG as an expectation over $p(\theta)p(y|\theta,\xi)$ is given by
\begin{equation}
	g_*(y,\theta,\xi) = \log\frac{p(y|\theta,\xi)}{p(y|\xi)}.
\end{equation}
Recall the score function identity
\begin{equation}
	\E_{p(x|\xi)}\left[ \ppx{\xi}\log p(x|\xi) \right] = 0.
\end{equation}
We have
\begin{align}
\E_{p(\theta)p(y|\theta,\xi)}&\left[\pypx{g_*}{\xi} \right] \\
&= \E_{p(\theta)p(y|\theta,\xi)}\left[\ppx{\xi}\log\frac{p(y|\theta,\xi)}{p(y|\xi} \right] \\
\begin{split}
&= \E_{p(\theta)}\left(\E_{p(y|\theta,\xi)}\left[\ppx{\xi}p(y|\theta,\xi) \right]\right) \\
&\quad- \E_{p(y|\xi)}\left[ \ppx{\xi}\log p(y|\xi) \right]
\end{split} \\
&=0
\end{align}
by two applications of the score function identity.
This suggests that, as $g$ becomes close to $g_*$, the $\partial g/\partial \xi$ term in \eqref{eq:reinforce} has expectation close to zero, and primarily contributes variance to the gradient estimator.

Theorem~\ref{lemma:lf} shows that if we remove the $\partial g/\partial \xi$ term, the resulting algorithm still optimizes a valid lower bound on $I(\xi)$. 
Specifically, removing this term is equivalent to the following gradient-coordinate algorithm. First, we choose the family $f_\psi(\theta,y)$ to be $p(y|\theta,\psi)$. Then at time step $t$ we do the following
\begin{enumerate}
	\item Set $\psi_t = \xi_t$
	\item Take a gradient step with respect to $(\xi, \phi)$ to update $\xi_t, \phi_t$
\end{enumerate}
Importantly, the new gradient does not include a $\partial g/\partial \xi$ term, but is the gradient of a valid lower bound on EIG. In practice, this alternative gradient did not yield substantially different performance from the standard approach of including the $\partial g / \partial \xi$ term. All our experiments used the standard approach for simplicity.

\section{EXPERIMENTS}

\subsection{Implementation}
All experiments were implemented in PyTorch 1.4.0 \citep{pytorch} and Pyro 0.3.4 \citep{pyro}. Supporting code can be found at \path{https://github.com/ae-foster/pyro/tree/sgboed-reproduce}, see `README.md` for details on how to run the experiments.

\subsection{Death process}
\label{app:death}

\begin{table}
	\vspace{-6pt}
	\begin{center}
		\caption{Death process. We present the final EIG for each method (computed using NMC with 200000 samples).}
		\label{tab:death_eig_full}
		{\renewcommand{\arraystretch}{1}
			\setlength\tabcolsep{4pt} 
			\begin{tabular}{lr}
			\hline
				 \small Method & \small EIG mean $\pm1$ s.e.  \\
				\hline
			    \small \textbf{ACE} & \small $\mathbf{0.9830 \pm 0.0001}$ \\
				 \small PCE & \small $0.9822\pm 0.0001$ \\
				 \small BA & \small $0.9822 \pm 0.0002$ \\
				 \small ACE without RB & \small $0.9789\pm 0.0006$ \\
				 \small PCE without RB & \small $0.9710\pm 0.0025$ \\
				 \small BA without RB & \small $0.9322\pm 0.0045$ \\
				\hline
				 \small BO with NMC & \small $0.9732 \pm 0.0009$ \\
				\hline
			\end{tabular}}
		\end{center}
	\end{table}

We place the prior $\theta \sim \text{LogNormal}(0, 1)$ on the infection rate and have the likelihood
\begin{align}
\begin{split}
	I_1 &\sim \text{Binomial}(N, e^{-\theta\xi_1}) \\
	I_2 &\sim \text{Binomial}(N - I_1, e^{-\theta\xi_2}).
\end{split}
\end{align}
We also have the constraint $\xi_1,\xi_2\ge 0$.

For each method, we fixed a computational budget of 120 seconds, and did 100 independent runs. For gradient methods, we used the Adam optimizer \citep{kingma2014adam} with learning rate $10^{-3}$ and the default momentum parameters. The inference network made a separate Gaussian approximation to the posterior for each of the 66 outcomes. To evaluate $I(\xi)$ for comparison we used NMC with a large number of samples: 20000 for Figure~\ref{fig:death_rb_bo} and 200000 for the final values in the caption and in Table~\ref{tab:death_eig_full}. For the BO, we used a Matern52 kernel with variance 1 and lengthscale 0.25, and the GP-UCB1 algorithm \citep{srinivas2009gaussian} for acquisition.

{We used the following number of samples for our Rao-Blackwellized estimators
\begin{center}
\begin{tabular}{lll}
	Method & Number of samples \\
	\hline
	ACE & 10 + 660 \\
	PCE & 10  \\
	BA & 10 \\
	\hline
	NMC & 2000 \\
\end{tabular}
\end{center}}

\subsection{Regression}
\label{app:regression}
We consider the following prior on $\theta = (\mathbf{w}, \sigma)$
\begin{align}
	w_j &\iid \text{Laplace}(1) \text{ for }j=1,...,p \\
	\sigma & \sim \text{Exponential}(1) 
\end{align}
with the likelihood
\begin{equation}
	y_i \sim N\left(\sum_{j=1}^p \xi_{ij}w_j, \sigma\right) \text{ for }i=1,...n.
\end{equation}
This represents a standard regression model, although with non-Gaussian prior distributions we cannot compute the posterior or true EIG analytically. To ensure the EIG has a finite maximum, we impose the following constraint
\begin{equation}
	\sum_j |\xi_{ij}| = 1 \text{ for }i=1,...,n.
\end{equation}
In practice, we set $n=p=20$.

For each of our five methods, we fixed the computational budget to 15 minutes and did 10 independent runs. For gradient methods, we used a learning rate of $10^{-3}$ and the Adam optimizer with default momentum parameters. The inference network used the following variational family
\begin{align}
	\mathbf{w} &\sim N(\bm{\mu}, s\Sigma_0) \\
	\sigma &\sim \Gamma(\alpha,\beta)
\end{align}
and we used a neural network with the following architecture
\begin{center}
\begin{tabular}{lll}
	Operation & Size & Activation \\
	\hline
	Input $\rightarrow$ H1 & 64 & ReLU \\
	H1 $\rightarrow$ H2 & 64 & ReLU \\
	H2 $\rightarrow \bm{\mu}$ & 20 & - \\
	H2 $\rightarrow (\alpha,\beta)$ & 2 & Softplus \\
	H2 $\rightarrow s$ & 1 & Softplus \\
	$\Sigma_0$ & $20 \times 20$ & - \\
\end{tabular}
\end{center}

\begin{figure}[t]
	\centering
	\begin{subfigure}[b]{0.40\textwidth}
		\centering
		\includegraphics[width=\textwidth]{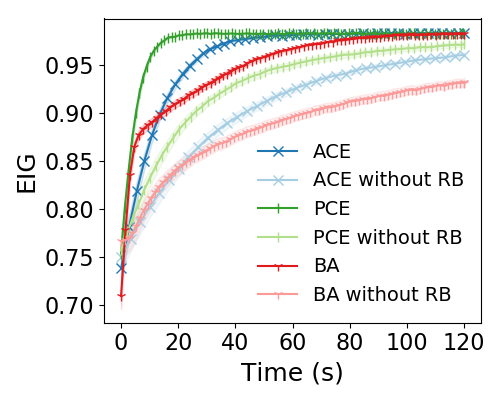}
	\end{subfigure}
	\vspace{-1pt}
	\caption{The EIG against time for the death process: comparing Rao-Blackwellization against no Rao-Blackwellization. Each method had a 120 second time budget.}
	\label{fig:death_rb_nrb}
\end{figure}

For BO and random search, point evaluations of $I(\xi)$ were made using VNMC. Each VNMC evaluation took 1000 steps, with the optimization as above (but with $\xi$ fixed). We used a GP with Matern52 kernel with lengthscale 5, variance 10. We used a GP-UCB1 acquisition rule, and terminated once 15 minutes had passed. For random search, we sampled designs using a standard unit Gaussian.

We used the following number of samples
\begin{center}
\begin{tabular}{lll}
	Method & Inner samples $L$ & Outer samples $N$ \\
	\hline
	ACE & 10 & 10 \\
	PCE & 10 & 10 \\
	BA & n/a & 100 \\
	\hline
	VNMC & 10 & 10 \\
\end{tabular}
\end{center}

To evaluate designs, we used ACE/VNMC. We first trained ACE using the same procedure as above, for 20000 steps. Then we made the final ACE/VNMC evaluations using the fixed inference network and $L=2.5\times10^{3}$ inner samples, $N=10^{5}$ outer samples.

\subsection{Advertising}
\label{app:largedim}
We introduce a LogNormal likelihood and a $D$-dimensional latent variable $\bm{\theta}$ governed by a Normal prior, the joint density of our model is 
\begin{equation}
p(\bm{y}, \bm{\theta} | \bm{\xi}) = \mathcal{LN}(\bm{y} | \bm{\theta} \odot \bm{\xi}, \sigma^2 \bm{\xi}) \mathcal{N}(\bm{\theta} | \bm{0}, \bm{\Lambda}_0)
\end{equation}
where $\sigma$ controls the observation noise, $\bm{\Lambda}_0$ is a non-diagonal precision matrix and $\odot$ denotes the Hadamard product. Since there are correlations
among the $D$ regions, the optimal advertising budget (w.r.t.~gaining information about $\bm{\theta}$) allocates more money to the regions that are tightly correlated. 

Throughout we assume that the number of regions $D$ is even.
We set the budget to scale with the number of dimensions, $B=\tfrac{D}{2}$, set $\sigma=1$ and choose
the prior precision matrix to be
\begin{equation}
\nonumber
\bm{\Lambda}_0 = (1+\tfrac{1}{D}) \mathbb{I}_D - \tfrac{1}{D}\bm{u} \bm{u}^{\rm T} \qquad \bm{u}^{\rm T}  
\equiv (\alpha, ...,\alpha, 1, ..., 1) \end{equation}
where the first $\tfrac{D}{2}$ components of $\bm{u}$ equal $\alpha$ and the last $\tfrac{D}{2}$ components equal $1$. We shall see that $\alpha=0.1$ controls the degree of asymmetry in the optimal design. Discarding an irrelevant constant,
we can compute the exact EIG using the formula:
\begin{equation}
\nonumber
I(\bm{\xi}) = \frac{1}{2 }\log\det \bm{\Lambda}_{\rm post} \qquad
\bm{\Lambda}_{\rm post}  = \bm{\Lambda}_0 + \frac{1}{\sigma^2}{\rm diag}(\bm{\xi})
\end{equation}
Using the matrix determinant lemma for rank-1 matrix updates we can then compute
\begin{equation}
\nonumber
\begin{split}
&{\rm log \; det\;} \bm{\Lambda}_{\rm post} = 
\sum_{i=1}^D \log (1 +  \tfrac{1}{D} + \xi_i) + \\
&\ \log \left(1 -  \sum_{i=1}^{\tfrac{D}{2}} \left \{ \tfrac{\alpha^2}{1 + \tfrac{1}{D} + \xi_i} \right\} 
- \sum_{i=1+\tfrac{D}{2}}^{D} \left \{ \tfrac{1}{1 + \tfrac{1}{D} + \xi_i} \right\} \right).
\end{split}
\end{equation}
By symmetry the optimum (it is easy to check that it is a maximum) of ${\rm EIG}(\bm{\xi})$ will 
satisfy $\xi_i = \xi_{i+1}$ for $i=1,...,\tfrac{D}{2}-1, \tfrac{D}{2}+1, ..., D$. In other words $\bm{\xi}$ is entirely specified
by $\xi_1$ and $\xi_D$, which must satisfy $\xi_1 + \xi_D = 1$ because of the constraint on the budget $B=\tfrac{D}{2}$.
Thus we have reduced the EIG maximization problem to a univariate optimization problem that can easily be solved to machine precision, for example by
gradient methods or brute force bisection. This analytic solution gives us the ground truth EIG, used within BO and for evaluation, and the true optimal design, used for evaluation.

For each of the four methods (ACE, PCE, BA and BO) we fix the computational budget to 120 seconds per design optimization. 
For the gradient-based methods this corresponds to $1\times 10^4$, $2\times 10^4$, and $1.8\times 10^4$ gradient steps for ACE, PCE, and BA, respectively. 
For the BO baseline, we run 110 steps of a GP-UCB-like algorithm \citep{srinivas2009gaussian} in batch-mode, 
resulting in a total budget of $1650$ function evaluations of the EIG oracle. 
Note that for all four methods the runtime dependence on the dimension $D$ is negligible in the regime in which we are operating; consequently we use the same number of gradient or BO steps for all $D$.

For the gradient-based methods, we use the Adam optimizer with default momentum hyperparameters
and an initial learning rate of $\ell_0 = 0.1$ that is exponentially
decayed towards a final learning rate $\ell_f$ that depends on the particular method. In particular we set 
$\ell_f = 1\times10^{-4}$, $\ell_f = 1 \times 10^{-5}$, and $\ell_f = 3 \times 10^{-4}$ for the ACE, PCE, and BA methods, respectively.
For the BO baseline, we used a Mat{\'e}rn kernel with a fixed length scale $\ell=0.2$. These hyperparameters were chosen by running
a grid search with $D=16$ and choosing hyperparameters that minimized the mean absolute EIG error.

Finally we note that in Fig.~\ref{fig:varydim} at each dimension $D$ we normalize the EIG by the factor
\begin{equation}
Z = {\rm EIG}(\bm{\xi}^*) - {\rm EIG}(\bm{\xi}_{\rm uniform}) 
\end{equation}
where $\bm{\xi}^*$ and $\bm{\xi}_{\rm uniform}$ are the optimal and uniform budget designs, respectively. Consequently
after normalization the absolute error for the uniform budget design $\bm{\xi}_{\rm uniform}$ is equal to 1.

\subsection{Biomolecular docking}
\label{app:docking}
For the docking model, we used the following independent priors
\begin{align}
	\text{top} & \sim \text{Beta}(25, 75) \\
	\text{bottom} & \sim \text{Beta}(4, 96) \\
	\text{ee50} & \sim N(-50, 15^2) \\
	\text{slope} & \sim N(-0.15, 0.1^2).
\end{align}
For the design $\xi = (\xi_1, ..., \xi_{100})$ we had 100 binary responses
\begin{equation}
	y_i \sim \text{Bern}\left( \text{bottom} + \frac{\text{top} - \text{bottom}}{1 + e^{-(\xi_i - \text{ee50}) \times \text{slope}}} \right).
\end{equation}

For gradient methods, we used the Adam optimizer with learning rate $10^{-3}$ and default momentum parameters. For each method, we took $5\times10^{5}$ gradient steps (each method converged within this number of steps). The inference network was mean-field with the same distributional families as the prior. We used the following neural architecture
\begin{center}
\begin{tabular}{lll}
	Operation & Size & Activation \\
	\hline
	Input $\rightarrow$ H1 & 64 & ReLU \\
	H1 $\rightarrow$ H2 & 64 & ReLU \\
	H2 $\rightarrow$ top & 2 & Softplus \\
	H2 $\rightarrow$ bottom & 2 & Softplus \\
	H2 $\rightarrow$ ee50 mean & 1 & - \\
	H2 $\rightarrow$ ee50 s.d. & 1 & Softplus \\
	H2 $\rightarrow$ slope mean & 1 & - \\
	H2 $\rightarrow$ slope s.d. & 1 & Softplus \\
\end{tabular}
\end{center}

We used the following number of samples
\begin{center}
\begin{tabular}{lll}
	Method & Inner samples $L$ & Outer samples $N$ \\
	\hline
	ACE & 10 & 10 \\
	PCE & 10 & 10 \\
	BA & n/a & 100 \\
\end{tabular}
\end{center}

For the expert method, the design of \citet{lyu2019ultra}, which comprised 580 compounds, was subsampled to comprise 100 compounds for a fair comparison.

For evaluation, we used ACE/VNMC, first training ACE for 25000 steps using the same learning rate as above. With the fixed inference network, we made ACE and VNMC evaluations using $L=2\times 10^3$ inner samples, $N=4\times 10^6$ outer samples.

\subsection{Constant elasticity of substitution}
\label{app:ces}
We used the exact set-up of \citet{foster2019variational}. Specifically, we take $U(\mathbf{x}) = \left(\sum_i x_i^\rho\alpha_i \right)^{1/\rho}$ and place the following priors on $\rho, \bm{\alpha},u$
\begin{align}
	\rho &\sim \text{Beta}(1,1) \\
	\bm{\alpha} & \sim \text{Dirichlet}([1, 1, 1]) \\
	\log u &\sim N(1, 3) \\
	\mu_\eta &= u \cdot (U(\mathbf{x}) - U(\mathbf{x}')) \\
	\sigma_\eta &= \tau u \cdot (1 + \|\mathbf{x} - \mathbf{x}'\|) \\
	\eta & \sim N(\mu_\eta, \sigma_\eta^2) \\
	y &= f(\eta)
\end{align}
where $f$ is the censored sigmoid function and $\tau=0.005$. All designs $\xi = (\mathbf{x},\mathbf{x}')$ were constrained to $[0, 100]^6$.

For gradient methods, we used the Adam optimizer with learning rate $10^{-3}$ and default momentum parameters. To make the design process 120 seconds per step, we used the following number of gradient steps
\begin{center}
\begin{tabular}{ll}
Method & Number of steps \\
\hline
ACE & 1500 \\
PCE & 2500 \\
BA & 5000
\end{tabular}
\end{center}
We found that there was insufficient time to effectively train a neural network guide. Instead we used a mean-field variational family with the same distributional families as the prior, and a linear model using the following features: $\text{logit}(y), \log| \text{logit}(y)|, \bm{1}(y > 0.5)$.

We used the following number of samples
\begin{center}
\begin{tabular}{lll}
	Method & Inner samples $L$ & Outer samples $N$ \\
	\hline
	ACE & 10 & 10 \\
	PCE & 10 & 10 \\
	BA & n/a & 100 \\
\end{tabular}
\end{center}

For the baseline, we used the marginal upper bound of \citet{foster2019variational} with the same variational family used in that paper---an $f$-transformed Normal with additional point masses at the end-points. We used a GP with a Mat{\'e}rn52 kernel, lengthscale 20, variance set from data, and a GP-UCB1 algorithm to make acquisitions which were done in batches of 8.

At each stage of the sequential experiment, the posterior was fitted using mean-field variational inference using the same distributional families as the prior.

\section{FUTURE WORK}
In this paper, we have focused on continuous design spaces in which gradient methods are applicable. One possible extension of our work would be to facilitate a unified one-stage approach to experimental design over \textit{discrete} design spaces. In this case, the lower bounds $I_{BA}, I_{ACE}$ and $I_{PCE}$ remains valid, and performing a joint maximization over $(\xi, \phi)$ on any of these objectives may be an attractive choice, although gradient optimization would no longer be appropriate for $\xi$. We envisage that one could apply existing methods for discrete optimization to the joint optimization problem over design and variational parameters. For instance, a continuous relaxation of the discrete variables, or MCMC-style updates on the discrete variables might be used. Future work might further explore this direction.


\end{document}